\documentclass{article}

\usepackage[preprint]{neurips_2026}

\usepackage{graphicx}
\usepackage{booktabs}
\usepackage{mathtools}
\usepackage{amsmath}
\usepackage{amssymb}
\usepackage{amsthm}          
\newtheorem{proposition}{Proposition}
\usepackage{algorithm}
\usepackage{algorithmic}
\usepackage{multirow}
\usepackage{xcolor}
\usepackage{hyperref}
\usepackage{cleveref}       

\newcommand{\method}{FreshPER}

\begin{document}

\title{Freshness-Aware Prioritized Experience Replay \\ for LLM/VLM Reinforcement Learning}

\author{
  \textbf{Weiyu Ma}\textsuperscript{1} \quad
  \textbf{Yongcheng Zeng}\textsuperscript{2} \quad
  \textbf{Yan Song}\textsuperscript{3} \quad
  \textbf{Xinyu Cui}\textsuperscript{2} \quad
  \textbf{Jian Zhao}\textsuperscript{4} \\
  \textbf{Xuhui Liu}\textsuperscript{1} \quad
  \textbf{Mohamed Elhoseiny}\textsuperscript{1}\thanks{Corresponding author.} \\[6pt]
  \textsuperscript{1}King Abdullah University of Science and Technology (KAUST) \\
  \textsuperscript{2}Chinese Academy of Sciences, Institute of Automation (CASIA) \\
  \textsuperscript{3}AI Centre, Department of Computer Science, University College London \\
  \textsuperscript{4}Zhongguancun Institute of Artificial Intelligence \\[4pt]
  \texttt{weiyu.ma@kaust.edu.sa} \quad \texttt{mohamed.elhoseiny@kaust.edu.sa} \\
}

\maketitle

\begin{abstract}
Reinforcement Learning (RL) has achieved impressive success in post-training Large Language Models (LLMs) and Vision-Language Models (VLMs), with on-policy algorithms such as PPO, GRPO, and REINFORCE++ serving as the dominant paradigm. 
However, these methods discard all collected trajectories after a single gradient update, resulting in poor sample efficiency, particularly wasteful for agentic tasks where multi-turn environment interactions are expensive. 
While Experience Replay drives sample efficiency in classic RL by allowing agents to reuse past trajectories and prioritize informative ones, directly applying Prioritized Experience Replay (PER) to LLMs fails. The rapid policy evolution of billion-parameter models renders stored priorities stale, causing old high-priority trajectories to dominate sampling long after they have become uninformative.
We propose \method{}, which addresses this \emph{priority staleness} problem by augmenting any PER-based priority with a multiplicative exponential age decay grounded in effective sample size analysis. To the best of our knowledge, \method{} is the first work to successfully apply PER to LLM/VLM reinforcement learning. We evaluate on eight multi-step agentic, reasoning, and math competition tasks with 0.5B, 3B, and 7B models. \method{} significantly outperforms on-policy baselines, achieving +46\% on NQ Search, +367\% on Sokoban, and +133\% on VLM FrozenLake, while standard PER without age decay consistently degrades performance. Our code is publicly available at \url{https://github.com/Vision-CAIR/Freshness-Aware-PER}.
\end{abstract}

\section{Introduction}
\label{sec:intro}
\begin{figure}[tb]
    \centering
    \includegraphics[width=\linewidth,trim=70 30 70 30,clip]{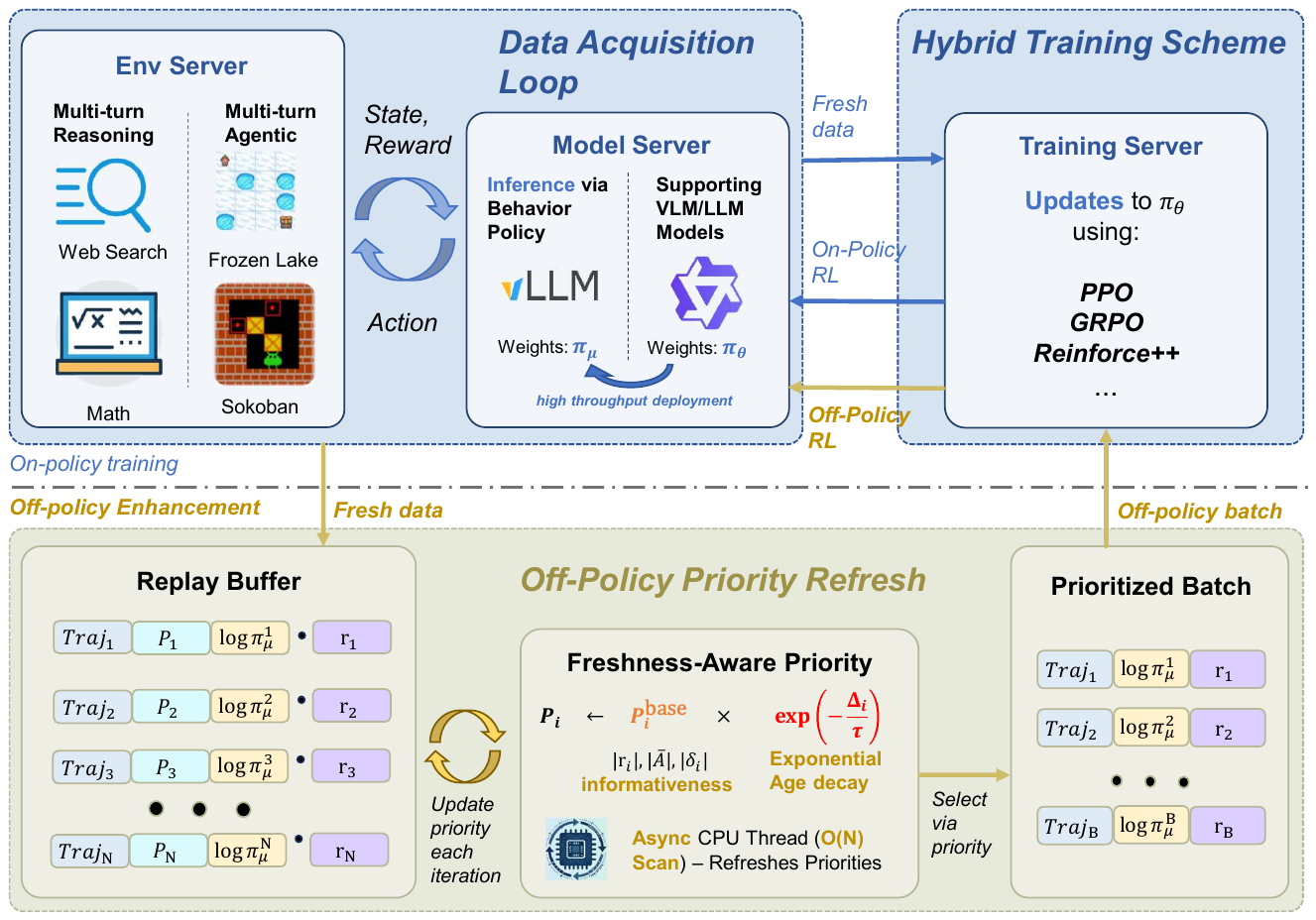}
    \caption{Overview of the \method{} training pipeline. \textbf{Top}: On-policy loop---the behavior policy $\pi_\mu$ (vLLM inference) interacts with agentic environments, and the current policy $\pi_\theta$ (DeepSpeed training) is updated via policy gradient on fresh data. \textbf{Bottom}: Off-policy loop---trajectories with their behavior log-probs and rewards are stored in the replay buffer on the CPU controller. An asynchronous thread refreshes priorities via $p_i \leftarrow p_i^{\text{base}} \cdot \exp(-\Delta_i/\tau)$, and prioritized batches are sampled for additional off-policy training.}
    \label{fig:pipeline}
\end{figure}
Reinforcement learning (RL) has become a transformative technique for large language models (LLMs). RLHF~\citep{ouyang2022training} played a central role in producing ChatGPT, demonstrating that RL-based post-training can dramatically improve the usability and safety of LLMs. More recently, OpenAI o1~\citep{openai2024o1} and DeepSeek-R1~\citep{deepseekr1} showed that RL can unlock advanced reasoning capabilities, achieving expert-level performance on mathematics and coding benchmarks. These successes have been powered by on-policy policy gradient algorithms---PPO~\citep{schulman2017ppo}, GRPO~\citep{shao2024deepseekmath}, and REINFORCE++~\citep{hu2025reinforce++}---which remain the dominant training paradigm.

A particularly exciting frontier is \emph{agentic RL}, where LLMs and vision-language models (VLMs) interact with external environments over multiple turns: searching the web~\citep{jin2025searchr1}, executing code~\citep{wei2025swerl}, calling tools~\citep{schick2023toolformer}, and navigating visual environments~\citep{driess2023palme}. Unlike single-turn preference alignment, agentic RL requires the model to take sequential actions and receive feedback from a live environment, bringing LLM training much closer to the classical RL paradigm. This shift also introduces a critical new challenge: \emph{environment interactions are very expensive}.

Consider training a search agent with REINFORCE++ on a retrieval-augmented QA task~\citep{jin2025searchr1}. Each prompt generates multiple rollout trajectories (e.g., 8 per prompt), and each trajectory involves up to 5 search turns. For a batch of 128 prompts, this yields over 5{,}000 retrieval calls per iteration, each requiring embedding computation and vector index lookup on dedicated hardware. The rollout stage alone dominates training time, often exceeding 70\% of the total wall-clock cost~\citep{rollflash2025}. Yet on-policy algorithms use these expensive trajectories for a single gradient update and then discard them entirely (\cref{fig:on_policy_llm}).

\begin{figure}[tb]
    \centering
    \includegraphics[width=0.85\linewidth]{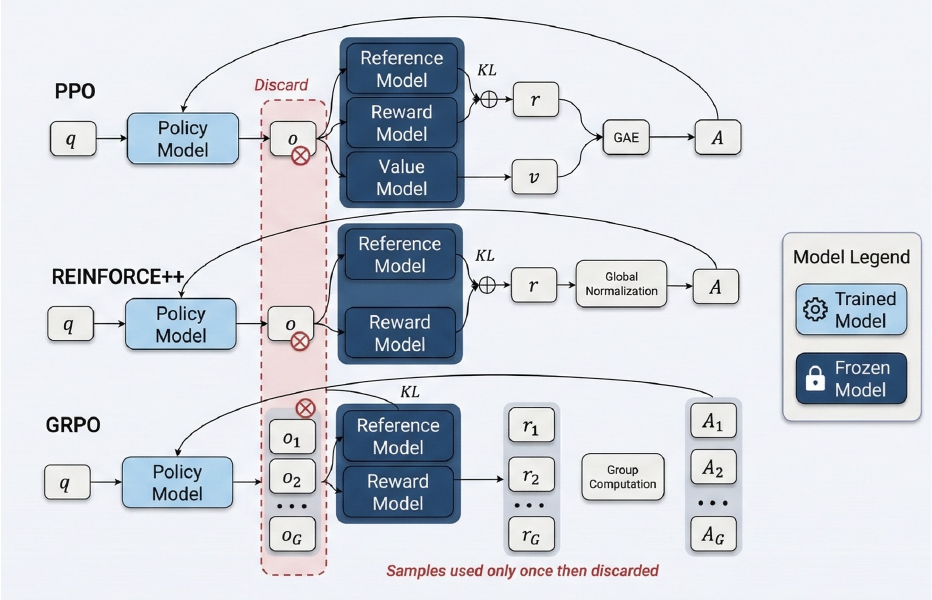}
    \caption{On-policy LLM RL algorithms (PPO, REINFORCE++, GRPO) use each trajectory for a single gradient update and then discard it ({\color{red}$\otimes$}), regardless of its learning potential.}
    \label{fig:on_policy_llm}
\end{figure}

In classical RL, \emph{experience replay}~\citep{lin1992self,mnih2015human} and its prioritized variant PER~\citep{schaul2016per} are the standard solutions to this problem, enabling agents to reuse past experiences and prioritize the most informative ones. However, directly applying PER to LLM RL fails. The core issue is \emph{priority staleness}: LLM policies evolve rapidly due to large gradient updates on long token sequences, causing old high-priority trajectories to dominate sampling long after they have become uninformative or even detrimental. Standard PER has no mechanism to account for this temporal degradation.

We present \method{}, which addresses priority staleness by augmenting any PER base priority with a multiplicative exponential age decay, directly motivated by the exponential decay of effective sample size (ESS) as the policy diverges from the behavior policy. This simple mechanism ensures that even the highest-priority old trajectory is eventually deprioritized below a fresh trajectory of moderate priority, while preserving the informativeness-driven sampling of standard PER.

The main contributions of this work are summarized as follows:
\begin{itemize}
    \item To the best of our knowledge, we are the first to successfully apply PER to LLM/VLM RL. We identify \emph{priority staleness} as a key failure mode and propose freshness-aware age decay grounded in importance sampling theory. 
    \item We implement a complete off-policy training pipeline with trajectory-level replay, integrated into the ROLL framework~\citep{wang2025roll}.
    \item We demonstrate consistent improvements across eight environments with 0.5B, 3B, and 7B models, achieving +46\% on NQ Search, +367\% on Sokoban, and +133\% on VLM FrozenLake, while standard PER without age decay consistently degrades.
\end{itemize}

\section{Related Work}
\label{sec:related}

Reinforcement learning has become a central technique for post-training LLMs and VLMs, spanning preference alignment~\citep{ouyang2022training}, reasoning~\citep{deepseekr1,shao2024deepseekmath}, agentic tasks~\citep{yao2023webshop,jin2025searchr1,wang2025ragen}, and visual reasoning~\citep{huang2025visionr1,shen2025vlmr1}. Distributed training frameworks such as ROLL~\citep{wang2025roll}, veRL~\citep{sheng2024verl}, and OpenRLHF~\citep{hu2024openrlhf} provide the infrastructure to scale these methods. Below we review three threads most relevant to our work.

\paragraph{Online and Offline RL for LLMs.}
Online (on-policy) methods dominate LLM RL. PPO~\citep{schulman2017ppo} remains the workhorse of RLHF~\citep{ouyang2022training}; GRPO~\citep{shao2024deepseekmath} simplifies it by removing the critic; DAPO~\citep{yu2025dapo}, REINFORCE++~\citep{hu2025reinforce++}, and Dr.~GRPO~\citep{liu2025drgrpo} further refine the policy gradient estimator. On the offline side, DPO~\citep{rafailov2023dpo} and its iterative variants~\citep{dong2024rlhfworkflow} optimize from preference data without environment interaction. A common limitation across online methods is that trajectories are discarded after a single gradient update, while offline methods forgo interaction entirely.

\paragraph{Off-Policy RL for LLMs.}
A middle ground is \emph{off-policy} training, which reuses historical trajectories. Asynchronous RLHF~\citep{noukhovitch2024async} and AReaL~\citep{areal2025} decouple generation from training with uniform replay. This introduces data staleness, prompting work on off-policy stability through importance-weight control~\citep{leroux2025topr,m2po2025,bapo2025,r2vpo2026}. On the data reuse side, RLEP~\citep{zhang2025rlep} replays correct trajectories, DOTS~\citep{dots2025} combines difficulty-targeted selection with replay, and LoRR~\citep{liu2025lorr} enables high replay ratios via parameter resets. Most recently, Fatemi~\citep{fatemi2026prioritized} proposes problem-level prioritized scheduling for RL post-training, but explicitly argues that transition-level PER is ``ill-suited for sequence models'' and instead uses success-rate-based curriculum scheduling. In contrast, we demonstrate that trajectory-level PER \emph{can} succeed in LLM RL when augmented with freshness-aware age decay. Across all existing approaches, none employ prioritized sampling that accounts for both sample informativeness and temporal freshness.

\paragraph{Experience Replay in Classic RL.}
Experience replay~\citep{lin1992self}, popularized by DQN~\citep{mnih2015human}, stores past transitions for reuse. PER~\citep{schaul2016per} assigns priority proportional to TD error so that ``surprising'' transitions are replayed more often, and has been extended to distributed settings~\citep{horgan2018distributed} and high replay ratios~\citep{doro2023sampleefficient,schwarzer2023bbf,fedus2020revisiting}. Most related to our decay mechanism is FPER~\citep{ma2022fper}, which discounts priority based on \emph{replay count}. In contrast, our age decay is measured in \emph{gradient steps} and directly grounded in the exponential ESS decay caused by policy divergence (\cref{sec:staleness}). These techniques are well-established for fixed-dimensional state-action spaces but have not been adapted to the LLM setting, where trajectories are variable-length token sequences and policy drift is far more rapid.

Despite recent attempts at problem-level scheduling~\citep{fatemi2026prioritized}, no prior work has successfully applied trajectory-level PER to LLM training. \method{} bridges this gap by combining informativeness-driven sampling with explicit temporal decay grounded in ESS analysis.

\section{Method: \method{}}
\label{sec:method}

\subsection{Problem Formulation}
\label{sec:formulation}
We model LLM RL as a multi-turn Markov Decision Process (MDP) $(\mathcal{S}, \mathcal{A}, T, R, \gamma)$. At each turn $t$, the state $s_t \in \mathcal{S}$ is the full conversation history, comprising the initial prompt concatenated with all prior assistant responses and environment observations. The action $a_t \in \mathcal{A}$ is the assistant response generated at the current turn. Upon receiving $a_t$, the environment returns an observation $o_t \in \mathcal{O}$. The transition function $T: \mathcal{S} \times \mathcal{A} \times \mathcal{O} \rightarrow \mathcal{S}$ is deterministic, defined by $s_{t+1} = s_t \oplus a_t \oplus o_t$, where $\oplus$ denotes sequence concatenation. The reward function $R: \mathcal{S} \times \mathcal{A} \rightarrow \mathbb{R}$ assigns a scalar reward to each state-action pair. The episode terminates after at most $H$ turns, and we set $\gamma = 1$, corresponding to an undiscounted episodic return.
Given a parameterized policy $\pi_\theta$, the objective is to maximize the expected discounted return:
\begin{equation}
    \max_\theta \; \mathbb{E}_{ \pi_\theta} \!\left[ \sum_t \gamma^tr_t \right] = \max_\theta \; \mathbb{E}_{ \pi_\theta} \!\left[ \sum_t r_t \right].
    \label{eq:objective}
\end{equation}
Policy gradient methods optimize this objective via clipped importance ratios $\rho = \pi_\theta(a \mid s) / \pi_{\mathrm{old}}(a \mid s)$, where $\pi_{\mathrm{old}}$ denotes the
current policy prior to the update. Accordingly, we define the state-value function $V_\pi(s_t) = \mathbb{E}_\pi[\sum_{k=0}^\infty \gamma^k r(s_{t+k}) \mid s_t]$ and the action-value function $Q_\pi(s_t, a_t)= \mathbb{E}_\pi[\sum_{k=0}^\infty \gamma^k r(s_{t+k}) \mid s_t,a_t]$, yielding the advantage function $A_\pi(s, a) \coloneqq Q_\pi(s, a) - V_\pi(s)$. 

\subsection{Prioritized Experience Replay}

Prioritized Experience Replay (PER)~\citep{schaul2016per} maintains a replay buffer $\mathcal{B}$ of transitions $(s, a, r, s')$ and samples them with probability proportional to their priority:
\begin{equation}
    P(i) = \frac{p_i^\alpha}{\sum_k p_k^\alpha}
    \label{eq:per_sampling}
\end{equation}
where $p_i$ is the priority of transition $i$ and $\alpha \in [0, 1]$ controls the degree of prioritization. When $\alpha = 0$, the sampling distribution reduces to uniform sampling. The priority is typically set to the absolute temporal-difference (TD) error: $p_i = |\delta_i| + \epsilon$, where the TD error $\delta_i = r + \gamma V_{\pi}(s') - V_{\pi}(s)$ measures the discrepancy between the observed return $r + \gamma V_{\pi}(s')$ and the current value estimate $V_{\pi}(s)$. Large $|\delta_i|$ indicates ``surprising'' transitions with high learning potential.

To correct for the non-uniform sampling bias, importance sampling weights are applied:
\begin{equation}
    w_i = \left( \frac{1}{N \cdot P(i)} \right)^\beta \bigg/ \max_j w_j
    \label{eq:is_weights}
\end{equation}
where $N = |\mathcal{B}|$ is the current number of transitions in the replay buffer, the $\max_j w_j$ normalization ensures weights are at most 1, and $\beta$ anneals from a small value to 1 over training, gradually increasing the correction.

\subsection{Priority Staleness and Exponential Age Decay}
\label{sec:staleness}
In standard Prioritized Experience Replay (PER)~\citep{schaul2016per}, each trajectory stored in the replay buffer carries a base priority $p_i^{\text{base}}$ that reflects its perceived learning value at the time of collection, such as the absolute TD error $|\delta_i|$, the reward magnitude $|r_i|$, or the absolute advantage $|\hat{A}_i|$.
Once assigned, this priority remains fixed.
As training progresses, however, the current policy $\pi_\theta$ gradually diverges from the behavior policy $\pi_\mu$ that originally generated the trajectory, and the priority recorded at collection time becomes an increasingly poor indicator of how useful the trajectory actually is for training the current policy.
We refer to this growing mismatch as \textbf{priority staleness}.

This problem is especially pronounced in LLM RL, where two factors amplify the effect.
First, because the policy contains billions of parameters, even a single gradient step can produce a substantial shift in the output distribution.
Second, the action space is discrete: each action corresponds to a token, and the probability of selecting a given token can change drastically after only a few updates.
Despite this rapid distributional shift, old trajectories with high priorities continue to dominate sampling long after they have become uninformative or even detrimental to learning.

When a trajectory collected under $\pi_\mu$ is reused to update $\pi_\theta$, the distributional mismatch between the two policies must be corrected through \emph{importance weighting}.
Concretely, each trajectory is weighted by the importance ratio
\begin{equation}
    \rho \;=\; \frac{\pi_\theta(a \mid s)}{\pi_\mu(a \mid s)},
    \label{eq:importance_ratio}
\end{equation}
which re-scales the contribution of each sample so that the resulting gradient estimate is unbiased with respect to $\pi_\theta$.
When the two policies are close, $\rho$ stays near~1 and the correction is benign.
As they diverge, however, $\rho$ becomes highly variable: a few samples receive very large weights while most receive weights close to zero.
In this regime, the gradient estimate is dominated by a small number of samples, and the majority of the replay buffer contributes little to learning.

The \emph{effective sample size} (ESS)~\citep{kong1992note} formalizes this intuition.
Given $n$ importance-weighted samples, the ESS quantifies how many samples drawn directly from $\pi_\theta$ would yield an estimate of equivalent statistical quality:
\begin{equation}
    \mathrm{ESS}
    \;=\; \frac{n}{1 + \mathrm{Var}_{\pi_\mu}[\rho]}.
    \label{eq:ess_def}
\end{equation}
Where $\mathrm{Var}[\rho] = \mathbb{E}[\rho^2] - (\mathbb{E}[\rho])^2$. 
When $\pi_\theta = \pi_\mu$, every importance weight equals~1, $\mathrm{Var}[\rho] = 0$, and all $n$ samples are fully effective ($\mathrm{ESS} = n$).
As the variance of $\rho$ increases, more samples are effectively wasted and the ESS shrinks.
Eq.~\eqref{eq:ess_def} reveals that the ESS is entirely governed by $\mathrm{Var}[\rho]$. To understand how fast the ESS decays as the policies diverge, it suffices to characterize how $\mathrm{Var}[\rho]$ grows over the course of training.
To connect $\mathrm{Var}[\rho]$ to a divergence measure, we introduce the $\chi^2$-divergence, defined as
\begin{equation}
    \chi^2(P \| Q)
    \;=\; \mathbb{E}_{Q}\!\left[\left(\frac{P(x)}{Q(x)} - 1\right)^{\!2}\right],
    \label{eq:chi2_def}
\end{equation}
which measures how far the likelihood ratio $P/Q$ deviates from~1 on average under~$Q$.
Since the importance ratio $\rho = \pi_\theta / \pi_\mu$ is precisely such a likelihood ratio, we can expand Eq.~\eqref{eq:chi2_def} as
\begin{equation}
    \chi^2\!\left(\pi_\theta \| \pi_\mu\right)
    \;=\; \mathbb{E}_{\pi_\mu}[\rho^2]
    - 2\,\mathbb{E}_{\pi_\mu}[\rho] + 1.
    \label{eq:chi2_expand}
\end{equation}
A fundamental property of importance weights is that their expectation under the proposal distribution is always~1:
\begin{equation}
    \mathbb{E}_{\pi_\mu}[\rho]
    \;=\; \int \pi_\mu(x)\,\frac{\pi_\theta(x)}{\pi_\mu(x)}\,dx
    \;=\; \int \pi_\theta(x)\,dx \;=\; 1,
    \label{eq:rho_mean_one}
\end{equation}
since $\pi_\theta$ is a valid probability distribution.
Substituting $\mathbb{E}_{\pi_\mu}[\rho] = 1$ into both Eq.~\eqref{eq:chi2_expand} and the standard variance formula $\mathrm{Var}[\rho] = \mathbb{E}[\rho^2] - (\mathbb{E}[\rho])^2$, both expressions reduce to $\mathbb{E}_{\pi_\mu}[\rho^2] - 1$, yielding the identity
\begin{equation}
    \mathrm{Var}_{\pi_\mu}[\rho]
    \;=\; \chi^2\!\left(\pi_\theta \| \pi_\mu\right).
    \label{eq:var_equals_chi2}
\end{equation}
This result converts the problem of bounding $\mathrm{Var}[\rho]$ into the problem of bounding the $\chi^2$-divergence between the two policies.

We now link the $\chi^2$-divergence back to the KL divergence.
The bridge is the R\'{e}nyi divergence of order~2, denoted $D_2$, which satisfies $D_2(P\|Q) = \log\mathbb{E}_Q[({P}/{Q})^2]$ by definition.
Combined with Eq.~\eqref{eq:var_equals_chi2}, this gives the exact relation $\chi^2 = \exp(D_2) - 1$~\citep{metelli2020importance}.
Moreover, R\'{e}nyi divergence is non-decreasing in its order, so $D_2 \geq D_{\mathrm{KL}}$.
Applying the monotonicity of the exponential function, we obtain
\begin{equation}
    \mathrm{Var}_{\pi_\mu}[\rho]
    \;=\; \exp(D_2) - 1
    \;\geq\; \exp(D_{\mathrm{KL}}) - 1.
    \label{eq:var_chain}
\end{equation}

Plugging this into the ESS formula, we obtain:
\begin{equation}
    \mathrm{ESS}
    \;\leq\; \frac{n}{1 + \exp(D_{{\mathrm{KL}}}) - 1}
    \;=\; n \cdot \exp(-D_{{\mathrm{KL}}}).
    \label{eq:ess_decay}
\end{equation}
Each sample's effective contribution to learning therefore decays \emph{exponentially} with the KL divergence between the two policies.

In this work, we heuristically approximate the degree of KL divergence between the two policies by $c \cdot \Delta$, where $c > 0$ is a constant determined by the learning rate and the typical gradient magnitude, and $\Delta$ denotes the number of gradient steps elapsed since collection. Substituting this approximation into Eq.~\eqref{eq:ess_decay} yields:
\begin{equation}
    \mathrm{ESS} \;\le\; n \cdot \exp(-c \cdot \Delta).
    \label{eq:ess_decay_final}
\end{equation}

Eq.~\eqref{eq:ess_decay_final} motivates down-weighting older trajectories by an exponential factor in their age.
In practice, the per-step divergence rate $c$ depends on training dynamics, such as learning rate schedule, gradient noise, and batch size, which are difficult to estimate reliably online.
We therefore absorb $c$ into a single tunable hyperparameter $\tau = 1/c$, which we call the \emph{age decay constant}, and define the age decay factor as
\begin{equation}
    w_{\mathrm{age}}(\Delta) \;=\; \exp\!\left(-\frac{\Delta}{\tau}\right).
    \label{eq:age_decay}
\end{equation}
When $\tau$ is large, the decay is gentle and older samples retain substantial weight; when $\tau$ is small, the decay is rapid and the buffer effectively forgets old data more quickly.
We compare this exponential form against polynomial decay alternatives in \cref{sec:analysis}.

\subsection{Freshness-Aware Priority}
\label{sec:fresh_priority}

Let $t_i$ be the global training step when trajectory $i$ was collected, and $t$ the current step. We define the \emph{age} $\Delta_i = t - t_i$ and propose the following priority function:

\begin{equation}
    p_i = \underbrace{p_i^{\text{base}}}_{\text{base priority}} \cdot \underbrace{\exp\!\left(-\frac{\Delta_i}{\tau}\right)}_{\text{age decay}},
    \label{eq:fresh_priority}
\end{equation}
where $p_i^{\text{base}}$ is an existing PER base priority and $\tau > 0$ is the age decay constant controlling the half-life of priority ($t_{1/2} = \tau \ln 2$). The base priority $p_i^{\text{base}}$ follows the classical PER intuition that higher-magnitude training signals indicate greater learning potential~\citep{schaul2016per}; it can be instantiated with any standard PER priority signal:

The base priority can be instantiated as reward magnitude $p_i^{\text{base}} = |r_i| + \epsilon$ for critic-free methods such as REINFORCE++~\citep{hu2025reinforce++} and GRPO~\citep{shao2024deepseekmath}, advantage magnitude $p_i^{\text{base}} = |\hat{A}_i| + \epsilon$ for actor-critic methods such as PPO~\citep{schulman2017ppo}, or TD error $p_i^{\text{base}} = |\delta_i| + \epsilon$ following the classical PER formulation~\citep{schaul2016per}. 

In all cases, the absolute value ensures that both high-reward and high-penalty trajectories receive high priority.
The age decay term $\exp(-\Delta_i / \tau)$ is a modular layer applied on top of any base priority, compensating for the exponential ESS decay in Eq.~\eqref{eq:ess_decay}. This multiplicative formulation decouples the question of \emph{what makes a sample informative} (base priority) from the question of \emph{how stale is a sample} (age decay), and ensures that even the highest-priority old trajectory is eventually deprioritized below a fresh trajectory of moderate priority.

\subsection{Sampling and Implementation Details}
\label{sec:sampling}

\paragraph{Proportional Sampling.}
Sampling from the priority distribution (Eq.~\eqref{eq:per_sampling}) na\"ively requires $O(N)$ time. We use a sum segment tree to achieve $O(\log N)$ per sample. To further reduce variance, we employ stratified sampling: the total priority mass $S = \sum_i p_i^\alpha$ is partitioned into $B$ equal segments (where $B$ is the batch size), and one trajectory is drawn uniformly from each segment.

\paragraph{Importance Sampling Correction.}
Prioritized sampling introduces bias that we correct with importance sampling weights:
\begin{equation}
    w_i = \left( \frac{1}{N \cdot P(i)} \right)^\beta \bigg/ \max_j w_j
    \label{eq:is_correction}
\end{equation}
where $N = |\mathcal{B}|$, $P(i) = p_i^\alpha / \sum_k p_k^\alpha$, and $\beta$ anneals from $\beta_0$ to 1.0 over training (following Eq.~\eqref{eq:is_weights}). These weights scale each sample's policy gradient loss $\ell_i$:
\begin{equation}
    \mathcal{L}_{\text{replay}} = \frac{1}{B} \sum_{i=1}^{B} w_i \cdot \ell_i
    \label{eq:weighted_loss}
\end{equation}

\paragraph{Priority Update.}
The base priority and the age decay term are updated independently. Each iteration, a background CPU thread refreshes the age decay for all entries by recomputing $\exp(-\Delta_i/\tau)$ with current ages (Line~6 of Algorithm~\ref{alg:freshper}). The base priority depends on the chosen instantiation: for reward-based priorities it is fixed at collection time, while for advantage- or TD-error-based variants it is recomputed after each training step to reflect the current policy.

\paragraph{Overall Algorithm.}
Algorithm~\ref{alg:freshper} summarizes the complete \method{} training procedure.

\begin{algorithm}[tb]
\caption{\method{} Training}
\label{alg:freshper}
\begin{algorithmic}[1]
\REQUIRE Policy $\pi_\theta$, environment $\mathcal{E}$, replay buffer $\mathcal{B}$, age decay constant $\tau$, replay ratio $K$
\FOR{each iteration}
    \STATE Roll out $\pi_\mu$ in $\mathcal{E}$ to collect episodes $\{e_j\}$ with rewards $\{r_j\}$
    \STATE Compute behavior log-probs $\log \pi_\mu(e_j)$ \hfill $\triangleright$ record before training
    \STATE Store each $e_j$ into $\mathcal{B}$ with base priority $p_j^{\text{base}}$ and $\log \pi_\mu$
    \STATE Train $\pi_\theta$ on fresh batch $\{e_j\}$ via policy gradient \hfill $\triangleright$ on-policy update
    \STATE Refresh priorities: $p_i \leftarrow p_i^{\text{base}} \cdot \exp(-\Delta_i / \tau)$ for all $i \in \mathcal{B}$ \hfill $\triangleright$ async on CPU
    \FOR{$k = 1, \ldots, K$}
        \STATE Sample batch $\mathcal{S}$ from $\mathcal{B}$ with $P(i) \propto p_i^\alpha$ \hfill $\triangleright$ stratified sampling
        \STATE Compute IS weights $w_i$ to correct sampling bias (Eq.~\eqref{eq:is_correction})
        \STATE Train $\pi_\theta$ on $\mathcal{S}$ with IS-weighted policy gradient
    \ENDFOR
    \STATE Sync inference engine: $\pi_\mu \leftarrow \pi_\theta$
\ENDFOR
\end{algorithmic}
\end{algorithm}

\cref{fig:pipeline} illustrates the overall training pipeline. The replay buffer and priority logic run on the CPU controller, naturally pipelining with GPU-bound inference and training. Further architectural details (eviction policy, asynchronous priority refresh, framework integration) are provided in the supplementary material.

\section{Experiments}
\label{sec:experiments}

We evaluate \method{} across eight environments spanning text-only LLM and multimodal VLM settings. All experiments use REINFORCE++~\citep{hu2025reinforce++} as the policy gradient algorithm and are implemented on the ROLL framework~\citep{wang2025roll}.

\subsection{Experimental Setup}

\paragraph{Environments.}
We select eight environments of varying difficulty and modality. On the \textbf{LLM} side, \textbf{NQ Search}~\citep{jin2025searchr1} is an agentic retrieval-augmented QA task on Natural Questions in which the model interleaves multi-step reasoning with search engine queries and produces a final answer scored by exact match; we use FAISS-based retrieval with E5 embeddings. \textbf{AIME} is a math competition task where the model solves American Invitational Mathematics Examination problems with integer answers (000--999), allowing up to 3 attempts per problem. \textbf{Sokoban} (ROLL built-in) is a box-pushing puzzle on a $6{\times}6$ grid requiring multi-step planning with irreversible actions; we evaluate a \emph{Simple} variant (larger rooms, fewer boxes) and a \emph{Hard} variant (tighter layouts, more boxes), with scores ranging from $-1$ to $+3$. \textbf{CliffWalking} (ROLL built-in) is a $4{\times}12$ grid navigation task (optimal score~$0$) that serves as a simple-environment control. \textbf{GSM8K}~\citep{cobbe2021gsm8k} consists of grade-school math word problems on which the base model already exceeds 93\% accuracy, serving as a near-saturated control. On the \textbf{VLM} side, \textbf{FrozenLake} (ROLL built-in) is a visual navigation task where a VLM agent navigates a $4{\times}4$ grid rendered as RGB images, avoiding holes on slippery ice; \textbf{GeoQA} is a geometry question-answering task requiring the VLM to interpret diagrams and solve problems.

\paragraph{Models.}
We use Qwen2.5-7B-Instruct~\citep{qwen2.5} for NQ Search and AIME, Qwen2.5-0.5B-Instruct for all other LLM tasks (Sokoban, CliffWalking, GSM8K), and Qwen2.5-VL-3B-Instruct~\citep{qwen2.5vl} for VLM tasks (FrozenLake, GeoQA). Full per-experiment hardware and configuration details are provided in the supplementary material.

\paragraph{RL Algorithm and Base Priority.}
Since REINFORCE++ is critic-free, we instantiate the base priority in Eq.~\eqref{eq:fresh_priority} as $p_i^{\text{base}} = |r_i| + \epsilon$ (reward magnitude). When combined with actor-critic methods, the base priority can be replaced with advantage- or TD-error-based variants without modifying the age decay mechanism.

\paragraph{Baselines.}
We compare three configurations: (1)~\textbf{On-Policy}---standard REINFORCE++ without a replay buffer, as implemented in ROLL~\citep{wang2025roll}; (2)~\textbf{Standard PER}---reward-based prioritized replay ($p_i = |r_i| + \epsilon$) without age decay, corresponding to conventional PER~\citep{schaul2016per} adapted for LLM RL; and (3)~\textbf{\method{} (Ours)}---freshness-aware PER with age decay ($p_i = (|r_i| + \epsilon) \cdot \exp(-\Delta_i / \tau)$, $\tau{=}500$ by default).

\paragraph{Implementation Details.}
We implement \method{} on the ROLL framework~\citep{wang2025roll}, using DeepSpeed for distributed training and vLLM for inference, with GPU counts ranging from 2 to 8 depending on model size (see supplementary material). Replay buffer capacity: 50K trajectories. Priority exponent $\alpha{=}0.6$, IS exponent $\beta{=}0.4$, age decay $\tau{=}500$ by default, replay ratio $K{=}2$, learning rate $1{\times}10^{-6}$, clip $\epsilon{=}0.2$. For off-policy training with smaller models (0.5B) and for AIME (7B), we use tighter advantage clipping and disable KL regularization (``Config~A''; see supplementary material). AIME uses $\tau{=}1000$.

\subsection{Main Results}
\label{sec:main_results}
\begin{figure}[htb]
\centering
\begin{minipage}[b]{0.32\linewidth}
    \centering
    \includegraphics[width=\linewidth]{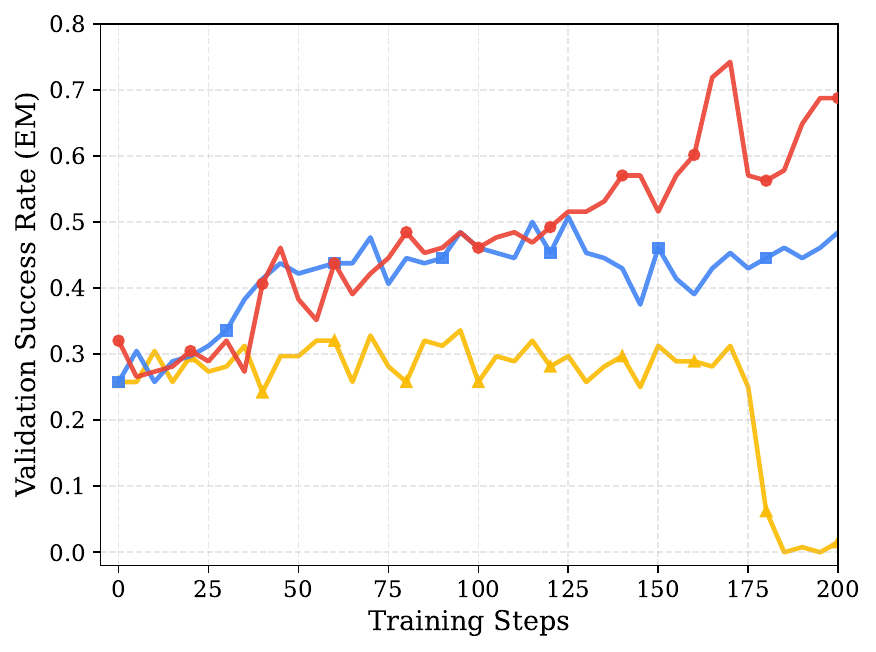}
    \\ {\small (a) NQ Search}
\end{minipage}
\hfill
\begin{minipage}[b]{0.32\linewidth}
    \centering
    \includegraphics[width=\linewidth]{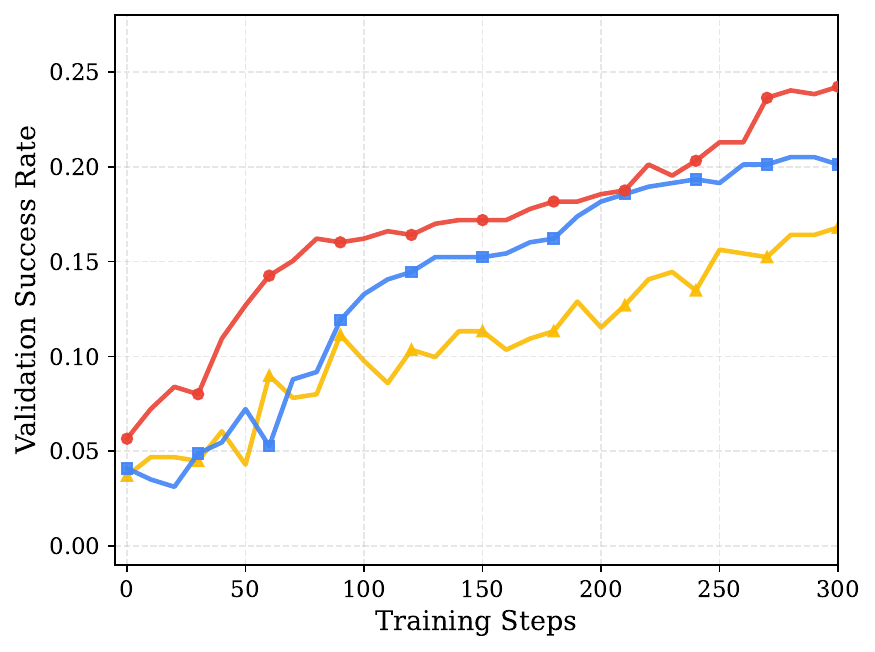}
    \\ {\small (b) AIME}
\end{minipage}
\hfill
\begin{minipage}[b]{0.32\linewidth}
    \centering
    \includegraphics[width=\linewidth]{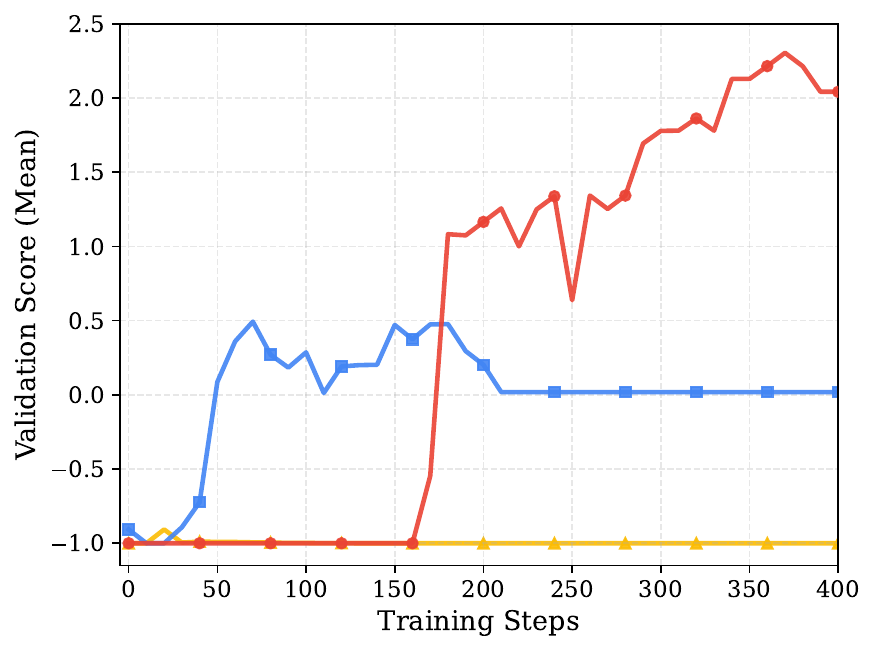}
    \\ {\small (c) Sokoban Simple}
\end{minipage}
\vspace{0.5em}
\hfill
\begin{minipage}[b]{0.32\linewidth}
    \centering
    \includegraphics[width=\linewidth]{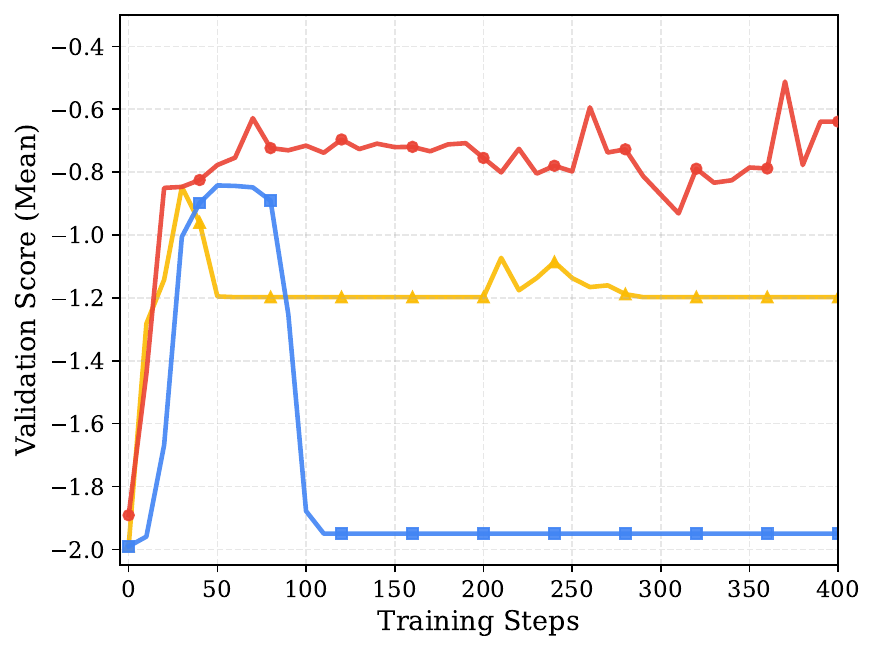}
    \\ {\small (d) Sokoban Hard}
\end{minipage}
\hfill
\begin{minipage}[b]{0.32\linewidth}
    \centering
    \includegraphics[width=\linewidth]{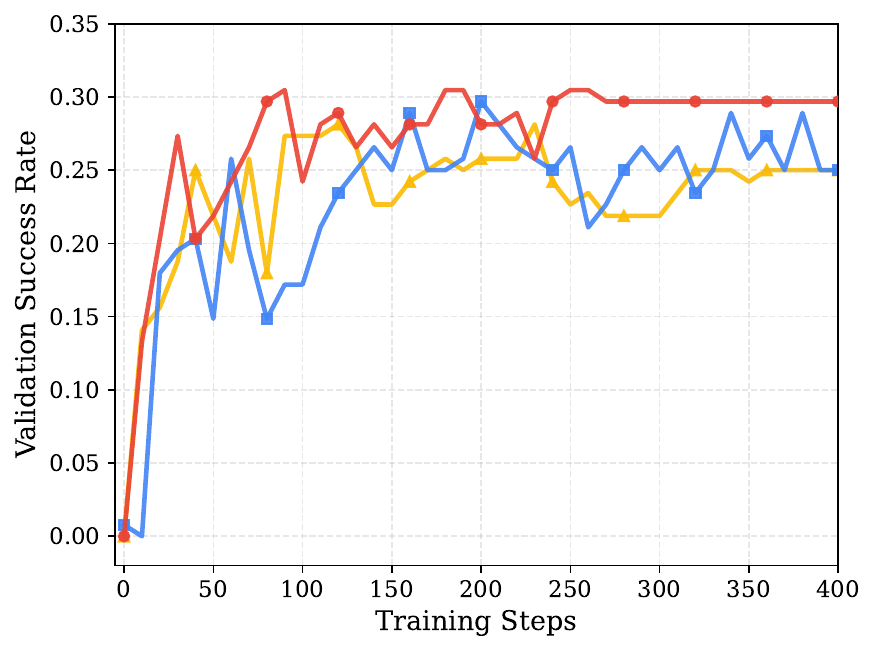}
    \\ {\small (e) FrozenLake}
\end{minipage}
\hfill\mbox{}
\caption{Learning curves on LLM tasks. {\color[HTML]{4285F4}Blue~$\blacksquare$}: On-Policy; {\color[HTML]{FBBC04}Yellow~$\blacktriangle$}: Standard PER; {\color[HTML]{EA4335}Red~$\bullet$}: \method{} (Ours).}
\label{fig:llm_main}
\end{figure}

\begin{figure}[tb]
\centering
\begin{minipage}[b]{0.48\linewidth}
    \centering
    \includegraphics[width=\linewidth]{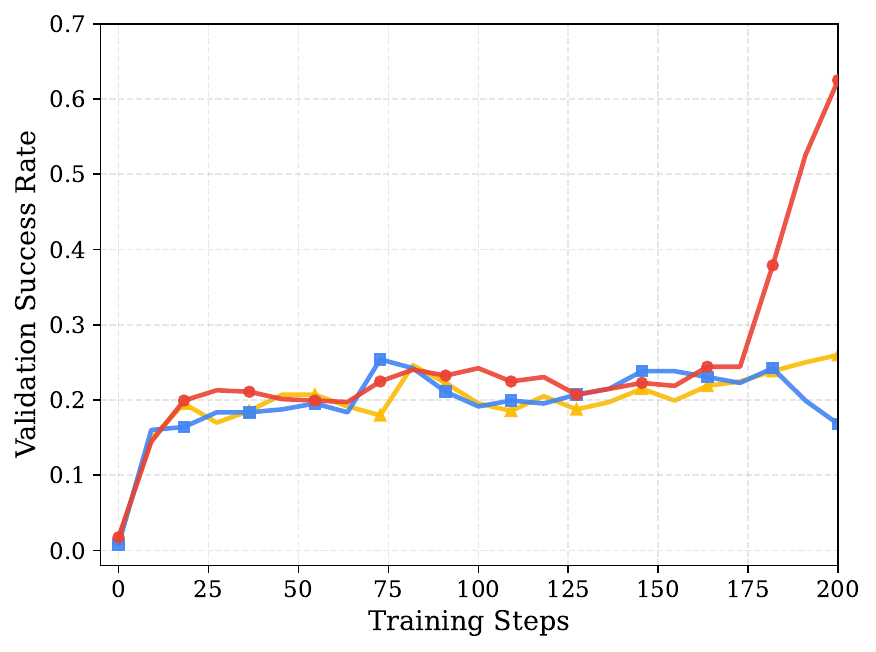}
    \\ {\small (a) VLM FrozenLake}
\end{minipage}
\hfill
\begin{minipage}[b]{0.48\linewidth}
    \centering
    \includegraphics[width=\linewidth]{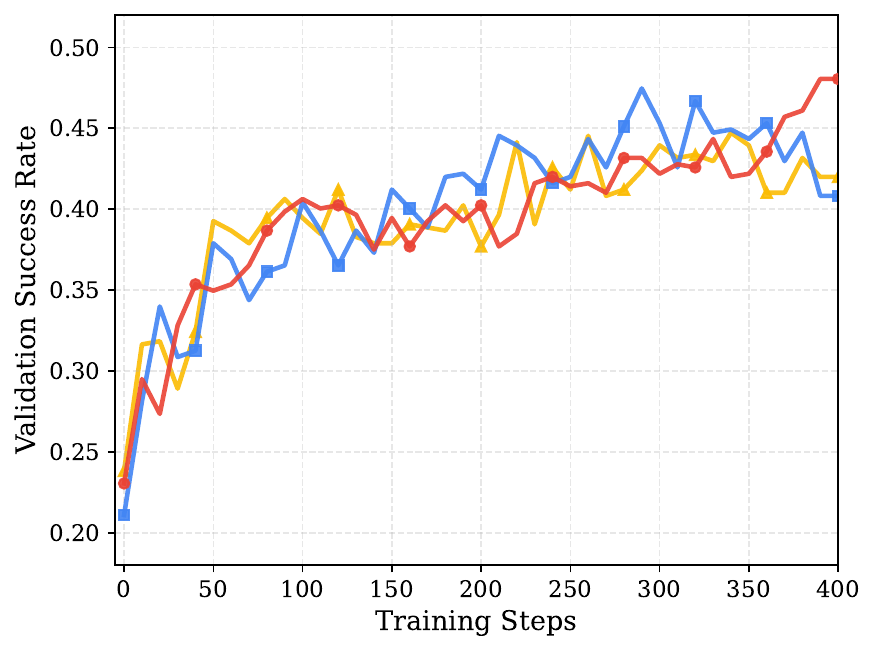}
    \\ {\small (b) VLM GeoQA}
\end{minipage}
\caption{Learning curves on VLM tasks. {\color[HTML]{4285F4}Blue~$\blacksquare$}: On-Policy; {\color[HTML]{FBBC04}Yellow~$\blacktriangle$}: Standard PER; {\color[HTML]{EA4335}Red~$\bullet$}: \method{} (Ours).}
\label{fig:vlm_main}
\end{figure}
\cref{tab:main_results} summarizes peak performance across all seven main experiments, and \cref{fig:llm_main,fig:vlm_main} show the corresponding learning curves. \method{} achieves the best peak performance on all seven tasks, with the largest gains on challenging agentic tasks (NQ Search +46\%, Sokoban Simple +367\%, VLM FrozenLake +133\%) and consistent improvements on math competition (AIME +18\%). Standard PER without age decay underperforms the on-policy baseline on most tasks, sometimes catastrophically (Sokoban Simple, NQ Search). We analyze these results in detail in \cref{sec:analysis}.

\begin{table}[htb]
    \caption{Peak performance across main tasks. Best in \textbf{bold}. \method{} achieves the best result on all tasks.}
    \label{tab:main_results}
    \centering
    \begin{tabular}{lllccc}
        \toprule
        Task & Type & Metric & On-Policy & Standard PER & \method{} (Ours) \\
        \midrule
        NQ Search       & LLM & EM (Success)  & 0.508 & 0.336 & \textbf{0.742} \\
        AIME            & LLM & Success       & 0.205 & 0.168 & \textbf{0.242} \\
        Sokoban Simple  & LLM & Score         & 0.493 & $-$0.907 & \textbf{2.304} \\
        Sokoban Hard    & LLM & Score         & $-$0.842 & $-$0.847 & \textbf{$-$0.512} \\
        FrozenLake      & LLM & Success       & 0.297 & 0.281 & \textbf{0.305} \\
        \midrule
        FrozenLake      & VLM & Success       & 0.270 & 0.250 & \textbf{0.630} \\
        GeoQA           & VLM & Success       & 0.475 & 0.447 & \textbf{0.481} \\
        \bottomrule
    \end{tabular}
\end{table}

\subsection{Analysis}
\label{sec:analysis}
\begin{figure}[htb]
\centering
\begin{minipage}[b]{0.48\linewidth}
    \centering
    \includegraphics[width=\linewidth]{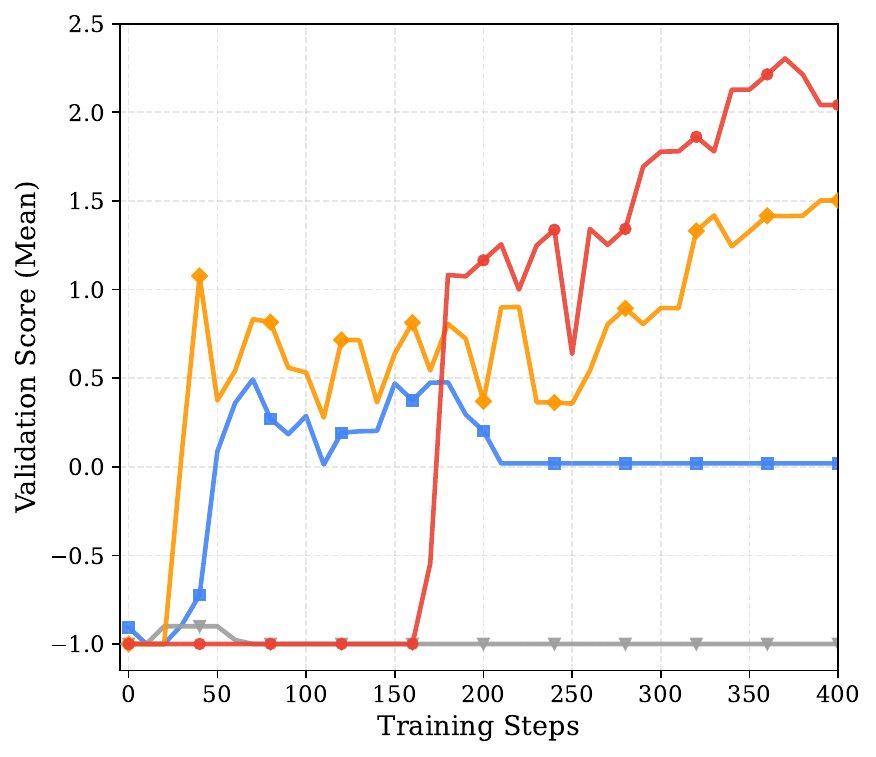}
    \\ {\small (a) Sokoban Simple}
\end{minipage}
\hfill
\begin{minipage}[b]{0.48\linewidth}
    \centering
    \includegraphics[width=\linewidth]{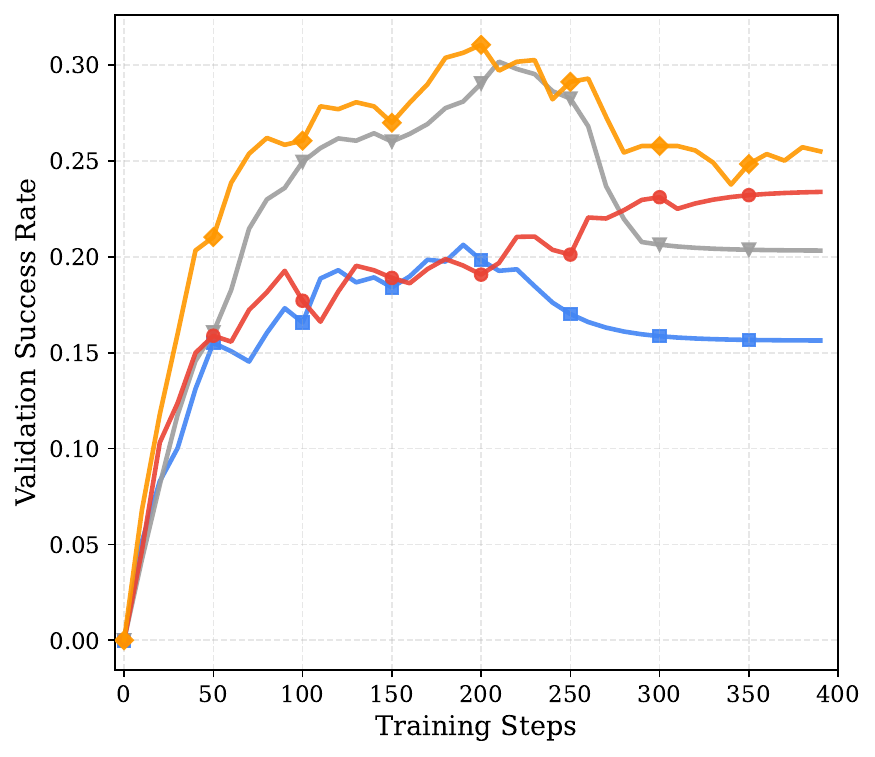}
    \\ {\small (b) FrozenLake (LLM)}
\end{minipage}
\caption{Ablation of age decay constant $\tau$. {\color[HTML]{4285F4}Blue~$\blacksquare$}: Baseline; {\color[HTML]{EA4335}Red~$\bullet$}: $\tau{=}500$; {\color[HTML]{FF9800}Orange~$\blacklozenge$}: $\tau{=}1000$; {\color[HTML]{9E9E9E}Gray~$\blacktriangledown$}: $\tau{=}1500$. (a)~Sokoban: $\tau{=}500$ is optimal; $\tau{=}1500$ fails completely. (b)~FrozenLake: $\tau{=}1000$ is optimal; all $\tau$ values outperform Baseline.}
\label{fig:tau_ablation}
\end{figure}
\paragraph{The benefit of replay scales with task difficulty.}

\begin{figure}[!htb]
\centering
\begin{minipage}[b]{0.48\linewidth}
    \centering
    \includegraphics[width=\linewidth]{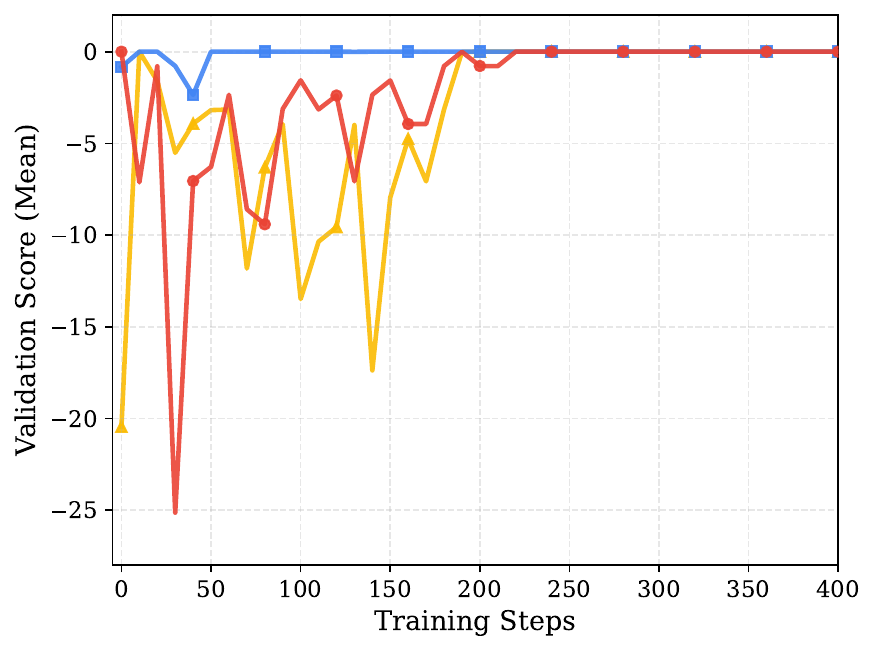}
    \\ {\small (a) CliffWalking}
\end{minipage}
\hfill
\begin{minipage}[b]{0.48\linewidth}
    \centering
    \includegraphics[width=\linewidth]{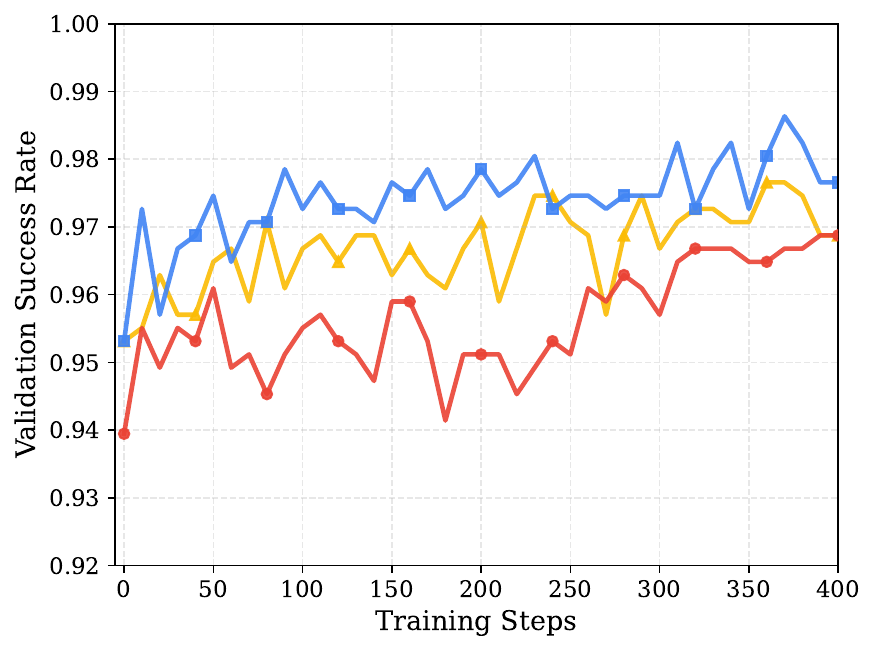}
    \\ {\small (b) GSM8K}
\end{minipage}
\caption{Control experiments. {\color[HTML]{4285F4}Blue~$\blacksquare$}: On-Policy; {\color[HTML]{FBBC04}Yellow~$\blacktriangle$}: Standard PER; {\color[HTML]{EA4335}Red~$\bullet$}: \method{} (Ours). (a)~CliffWalking: all methods converge to optimal; replay adds transient instability. (b)~GSM8K: initial performance already ${>}$93\%; all methods saturate at ${\sim}$97\%.}
\label{fig:control}
\end{figure}

Ablation on the age decay constant $\tau$ across Sokoban Simple and FrozenLake (\cref{fig:tau_ablation}) reveals that $\tau$ is task-dependent: Sokoban, which exhibits rapid policy drift, requires aggressive decay ($\tau{=}500$), while the slower-evolving FrozenLake benefits from a gentler setting ($\tau{=}1000$). This pattern extends to the main results. The largest gains appear on tasks where on-policy training struggles or collapses: NQ Search (+46\%), Sokoban Simple (+367\%), and VLM FrozenLake (+133\%). Conversely, control experiments on CliffWalking and GSM8K (\cref{fig:control}) show that when the task is simple enough for on-policy training to solve quickly, or the model is already near-saturated, replay provides little additional benefit. Taken together, these results suggest a practical guideline: \emph{the harder the task and the faster the policy evolves, the more valuable freshness-aware replay becomes, and the more aggressively $\tau$ should be set}.

\paragraph{Priority staleness validates the need for age decay.}
Standard PER without age decay underperforms the on-policy baseline on most tasks, and fails catastrophically on Sokoban Simple and NQ Search. This is consistent with the ESS decay analysis in \cref{sec:staleness}: without temporal decay, old trajectories with high base priorities dominate sampling even after they have become stale relative to the current policy. The contrast between Standard PER and \method{} isolates the contribution of age decay, confirming that it is essential for applying PER to LLM RL.

\paragraph{Training stability.}

\begin{figure}[htb]
    \centering
    \includegraphics[width=0.55\linewidth]{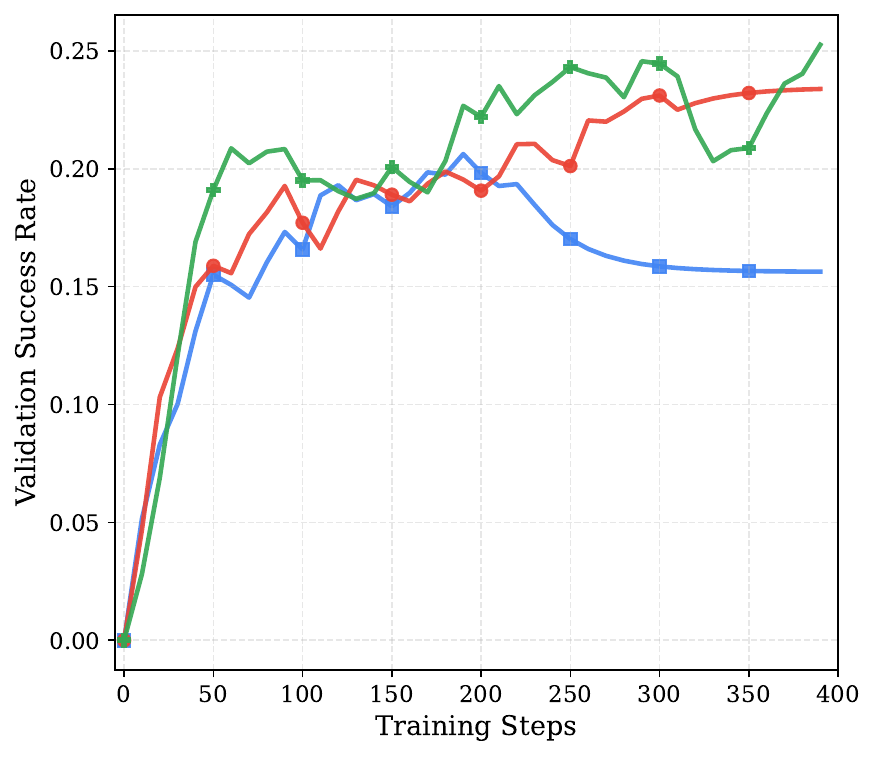}
    \caption{IS correction ablation on FrozenLake (LLM). {\color[HTML]{4285F4}Blue~$\blacksquare$}: Baseline; {\color[HTML]{EA4335}Red~$\bullet$}: $\tau{=}500$; {\color[HTML]{34A853}Green~$\boldsymbol{+}$}: $\tau{=}500$ + IS. Adding IS ($\beta{=}0.4$) does not substantially improve peak performance but eliminates late-stage degradation: the IS variant maintains its peak (0.281) through the end of training.}
    \label{fig:is_ablation}
\end{figure}
Across multiple tasks, on-policy training peaks early and then degrades (e.g., Sokoban Simple, GeoQA). \method{} consistently sustains its peak performance through the end of training. IS correction further enhances this stability without improving peak performance (\cref{fig:is_ablation}), validating the modular design where age decay and IS correction address complementary aspects of off-policy learning.

\paragraph{Generalization to VLM.}
\method{} transfers to multimodal settings without modification, achieving +133\% on VLM FrozenLake. This confirms that the age decay mechanism is agnostic to the observation modality.

\section{Conclusion}
\label{sec:conclusion}

We presented \method{}, which, to the best of our knowledge, is the first to successfully apply Prioritized Experience Replay to LLM and VLM reinforcement learning. The key enabler is freshness-aware age decay, which addresses the priority staleness problem caused by the rapid policy evolution of billion-parameter models. Evaluations across eight environments with models ranging from 0.5B to 7B confirm that \method{} consistently outperforms both on-policy baselines and standard PER, with particularly large gains on challenging agentic tasks. More broadly, our results suggest that classical sample-efficiency techniques from RL remain highly effective in the LLM era, provided they are adapted to handle the unique training dynamics of large language models.

\paragraph{Limitations and Future Work.}
Our experiments span three model scales (0.5B, 3B, and 7B parameters), but the interaction of freshness decay with even larger models (e.g., 70B+) and different training dynamics remains to be studied. Automatic tuning of $\tau$ based on observed policy divergence rates is a promising direction for future work.

\section*{Acknowledgement}
This research was supported by funding from the King Abdullah University of Science and Technology (KAUST) Center of Excellence for Generative AI under Award No. 5940.
\bibliographystyle{plainnat}
\bibliography{main}

\newpage
\appendix

\section{System Design Details}
\label{app:system}

\subsection{Replay Buffer Architecture}
\label{app:buffer_arch}

The replay buffer operates at the \emph{trajectory level}: each entry stores a complete episode (prompt, all assistant turns, and environment observations) along with its behavior log-probabilities, reward, collection step~$t_i$, and current priority~$p_i$. This granularity naturally matches agentic RL, where episode-level rewards are the primary training signal.

When the buffer reaches capacity, we evict the entry with the lowest effective priority---i.e., the oldest, lowest-priority trajectory---which follows naturally from the age decay: entries whose priority has decayed near zero are evicted first. The priority refresh step (Line~5 of Algorithm~\ref{alg:freshper} in the main paper) recomputes all priorities with current ages once per iteration. To avoid blocking the training loop, this $O(N)$ scan runs in a background CPU thread during on-policy GPU training, completing in $50$--$100$\,ms for $N{=}100$K entries and fully overlapping with GPU computation.

\subsection{Framework Integration}
\label{app:pipeline}

We implement \method{} on top of the ROLL framework~\citep{wang2025roll}, which provides a single-controller architecture that decouples inference and training across GPU pools. The replay buffer and all priority logic run on the CPU-side controller, naturally pipelining with GPU-bound training and inference. The training flow follows Algorithm~\ref{alg:freshper}: fresh rollouts are first used for on-policy training, then stored into the buffer; the replay loop draws prioritized batches for $K$ additional off-policy updates per environment interaction.

\section{Additional Ablation Studies}
\label{app:ablations}

\subsection{Age Decay Constant $\tau$}
\label{app:tau_ablation}

We ablate $\tau \in \{500, 1000, 1500\}$ on Sokoban Simple and FrozenLake (LLM, 0.5B), comparing against On-Policy and Standard PER (equivalent to $\tau{=}\infty$).

\begin{figure}[htb]
\centering
\begin{minipage}[b]{0.48\linewidth}
    \centering
    \includegraphics[width=\linewidth]{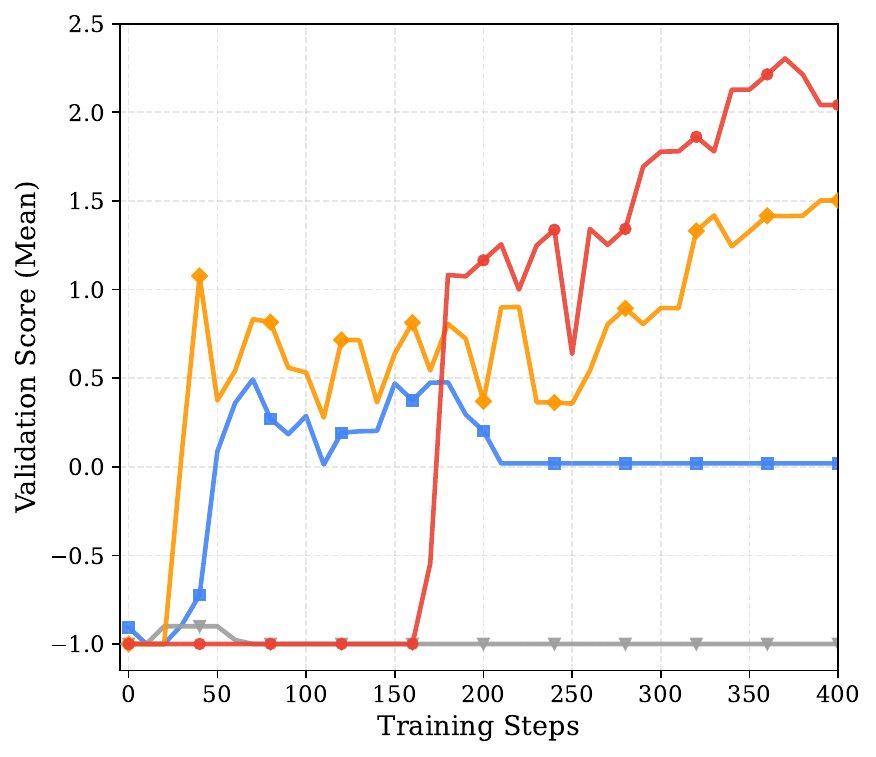}
    \\ {\small (a) Sokoban Simple}
\end{minipage}
\hfill
\begin{minipage}[b]{0.48\linewidth}
    \centering
    \includegraphics[width=\linewidth]{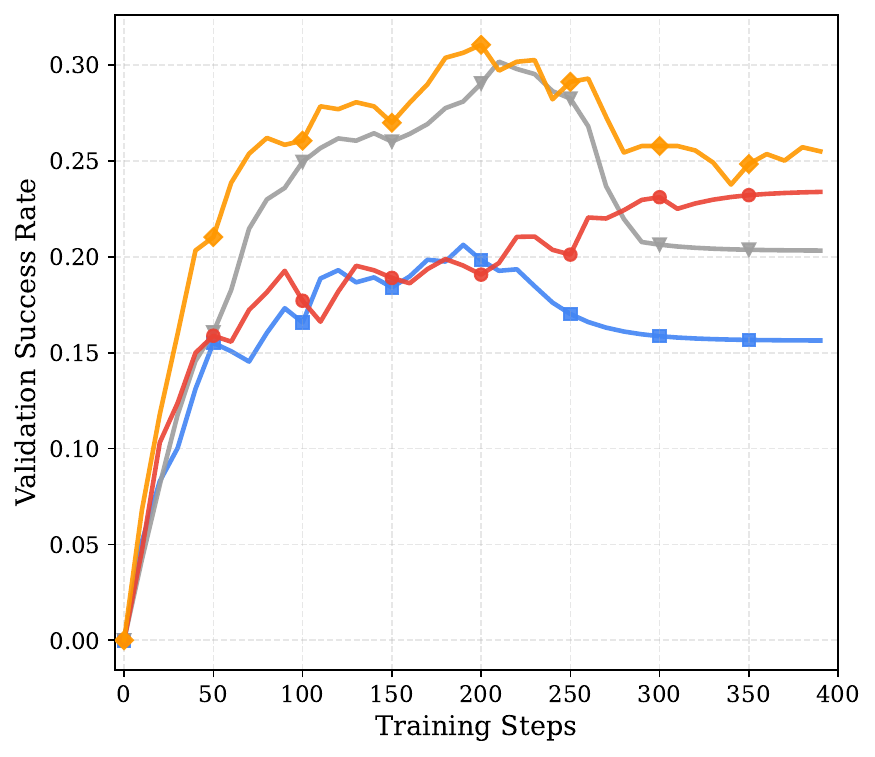}
    \\ {\small (b) FrozenLake (LLM)}
\end{minipage}
\caption{Ablation of age decay constant $\tau$. {\color[HTML]{4285F4}Blue~$\blacksquare$}: Baseline; {\color[HTML]{EA4335}Red~$\bullet$}: $\tau{=}500$; {\color[HTML]{FF9800}Orange~$\blacklozenge$}: $\tau{=}1000$; {\color[HTML]{9E9E9E}Gray~$\blacktriangledown$}: $\tau{=}1500$.}
\label{fig:supp_tau_ablation}
\end{figure}

On Sokoban Simple (\cref{fig:supp_tau_ablation}a), $\tau{=}500$ achieves a peak score of 2.30, $\tau{=}1000$ reaches 1.50, and $\tau{=}1500$ completely fails ($-$0.90), identical to Standard PER. This environment exhibits rapid policy drift, requiring aggressive decay. On FrozenLake (\cref{fig:supp_tau_ablation}b), the ranking shifts: $\tau{=}1000$ achieves the highest peak (0.336), followed by $\tau{=}1500$ (0.328) and $\tau{=}500$ (0.266). Unlike Sokoban, even $\tau{=}1500$ remains effective, indicating that FrozenLake tolerates higher data staleness.

\subsection{Importance Sampling Correction}
\label{app:is_ablation}

We investigate whether importance sampling (IS) correction improves training stability when combined with freshness decay. We compare $\tau{=}500$ with and without IS ($\beta{=}0.4$) on FrozenLake.

\begin{figure}[htb]
    \centering
    \includegraphics[width=0.55\linewidth]{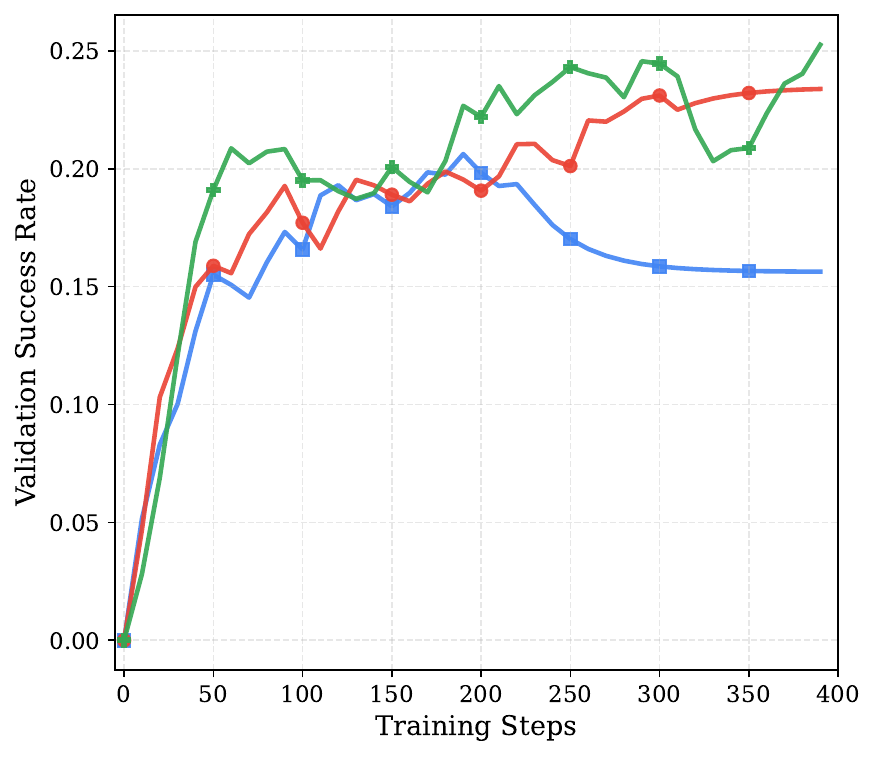}
    \caption{IS correction ablation on FrozenLake (LLM). {\color[HTML]{4285F4}Blue~$\blacksquare$}: Baseline; {\color[HTML]{EA4335}Red~$\bullet$}: $\tau{=}500$; {\color[HTML]{34A853}Green~$\boldsymbol{+}$}: $\tau{=}500$ + IS.}
    \label{fig:supp_is_ablation}
\end{figure}

\cref{fig:supp_is_ablation} shows that IS correction does not substantially boost peak performance ($+$5.9\%, from 0.266 to 0.281), but its primary value lies in \emph{training stability}: the IS variant is the only configuration whose final performance equals its peak (0.281 at both step~40 and step~390), whereas all other methods degrade in late training. This complementarity---freshness decay improves peak performance by suppressing stale priorities, while IS correction stabilizes training by compensating for distribution shift---validates the modular design of \method{}, where age decay and IS correction can be independently activated based on task requirements.

\subsection{When Does Replay Help?}
\label{app:control}

To understand the boundary conditions of \method{}, we evaluate on two environments where replay is expected to offer little benefit.

\begin{figure}[htb]
\centering
\begin{minipage}[b]{0.48\linewidth}
    \centering
    \includegraphics[width=\linewidth]{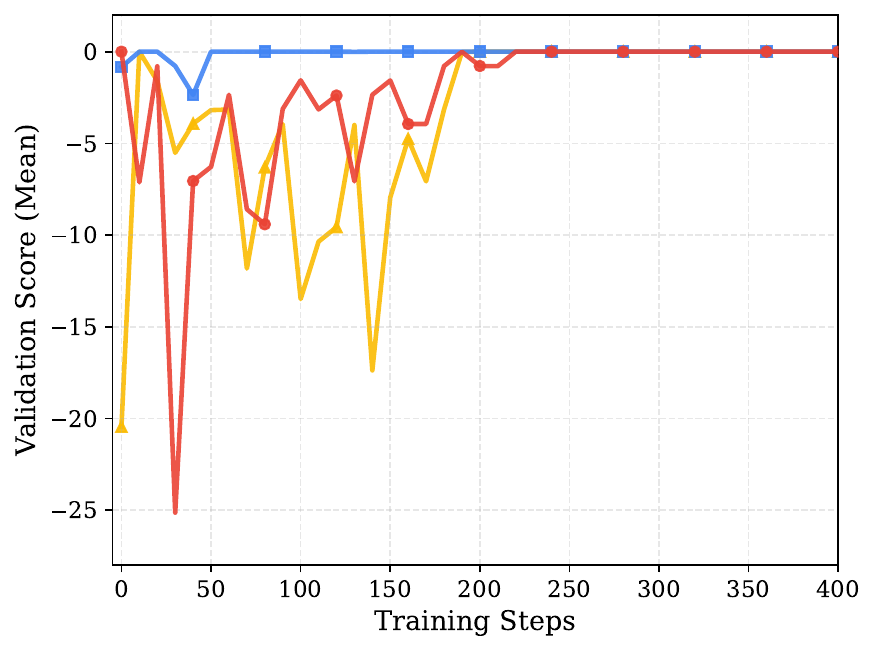}
    \\ {\small (a) CliffWalking}
\end{minipage}
\hfill
\begin{minipage}[b]{0.48\linewidth}
    \centering
    \includegraphics[width=\linewidth]{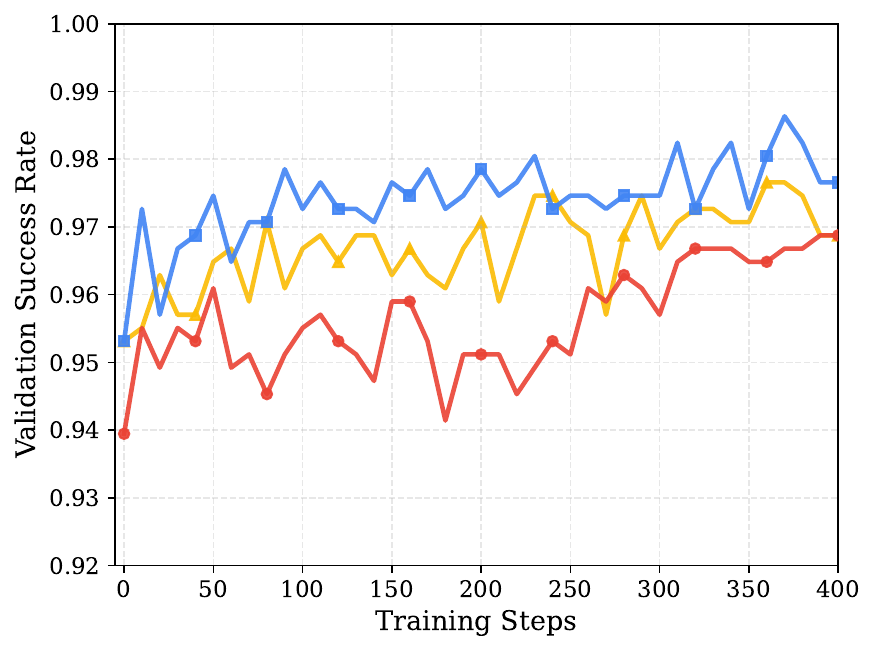}
    \\ {\small (b) GSM8K}
\end{minipage}
\caption{Control experiments. {\color[HTML]{4285F4}Blue~$\blacksquare$}: On-Policy; {\color[HTML]{FBBC04}Yellow~$\blacktriangle$}: Standard PER; {\color[HTML]{EA4335}Red~$\bullet$}: \method{} (Ours).}
\label{fig:supp_control}
\end{figure}

\paragraph{Too-Simple Environments (CliffWalking).}
All three methods converge to the optimal score of 0 (\cref{fig:supp_control}a). On-Policy converges fastest and most stably, while Standard PER and \method{} exhibit transient instability (PER shows spikes down to $-$17 at step~140). When the task is simple enough for on-policy training to solve quickly, replay provides no benefit and may introduce unnecessary off-policy noise.

\paragraph{Near-Saturated Models (GSM8K).}
All methods cluster in the narrow 0.94--0.99 range (\cref{fig:supp_control}b). The base model already achieves 93--95\% success; training provides only ${\sim}$2--3\% additional improvement. With so little room for improvement, replay-based methods cannot demonstrate meaningful advantages.

These control experiments delineate the regime where \method{} is most valuable: tasks that are \emph{challenging enough} that on-policy training either converges slowly or collapses, and where the model has \emph{sufficient room for improvement}.

\section{Derivations for Priority Staleness}
\label{app:staleness_derivation}

This section provides detailed derivations for the results stated in \cref{sec:staleness}.

\subsection{Variance of Importance Weights Equals \texorpdfstring{$\chi^2$}{Chi-squared}-Divergence}
\label{app:var_chi2}

\begin{proposition}
\label{prop:var_chi2}
Let $P$ and $Q$ be two probability distributions with $Q$ absolutely continuous with respect to $P$, and let $\rho(x) = P(x)/Q(x)$ denote the importance ratio. Then
\begin{equation}
    \mathrm{Var}_{Q}[\rho] = \chi^2(P \| Q),
\end{equation}
where $\chi^2(P\|Q) = \mathbb{E}_Q\!\left[(\rho - 1)^2\right]$ is the $\chi^2$-divergence.
\end{proposition}

\begin{proof}
By the definition of variance:
\begin{equation}
    \mathrm{Var}_{Q}[\rho]
    = \mathbb{E}_{Q}[\rho^2] - \left(\mathbb{E}_{Q}[\rho]\right)^2.
    \label{eq:app_var_def}
\end{equation}
We first evaluate the mean of $\rho$ under $Q$:
\begin{equation}
    \mathbb{E}_{Q}[\rho]
    = \int Q(x) \cdot \frac{P(x)}{Q(x)} \, dx
    = \int P(x) \, dx
    = 1,
    \label{eq:app_rho_mean}
\end{equation}
where the last equality holds because $P$ is a valid probability distribution. Substituting into Eq.~\eqref{eq:app_var_def}:
\begin{equation}
    \mathrm{Var}_{Q}[\rho]
    = \mathbb{E}_{Q}[\rho^2] - 1.
    \label{eq:app_var_simplified}
\end{equation}

Now expand the $\chi^2$-divergence from its definition:
\begin{equation}
    \chi^2(P \| Q)
    = \mathbb{E}_{Q}\!\left[(\rho - 1)^2\right]
    = \mathbb{E}_{Q}[\rho^2] - 2\,\mathbb{E}_{Q}[\rho] + 1
    = \mathbb{E}_{Q}[\rho^2] - 1,
    \label{eq:app_chi2_expanded}
\end{equation}
where we again used $\mathbb{E}_Q[\rho] = 1$. Comparing Eq.~\eqref{eq:app_var_simplified} and Eq.~\eqref{eq:app_chi2_expanded} completes the proof.
\end{proof}

\subsection{Monotonicity of R\'{e}nyi Divergence and the \texorpdfstring{$D_2 \geq D_{\mathrm{KL}}$}{D2 >= DKL} Bound}
\label{app:renyi_mono}

The R\'{e}nyi divergence of order $\alpha > 0$ ($\alpha \neq 1$) between distributions $P$ and $Q$ is defined as:
\begin{equation}
    D_\alpha(P \| Q)
    = \frac{1}{\alpha - 1}\,\log \mathbb{E}_{Q}\!\left[\left(\frac{P(x)}{Q(x)}\right)^{\!\alpha}\right].
    \label{eq:app_renyi_def}
\end{equation}
Two standard properties are relevant here:

\paragraph{Property 1: Connection to $\chi^2$-divergence.}
Setting $\alpha = 2$ in Eq.~\eqref{eq:app_renyi_def}:
\begin{equation}
    D_2(P \| Q)
    = \log \mathbb{E}_{Q}\!\left[\rho^2\right]
    = \log\!\left(1 + \chi^2(P \| Q)\right),
    \label{eq:app_d2_chi2}
\end{equation}
where we used $\mathbb{E}_Q[\rho^2] = 1 + \chi^2$ from Eq.~\eqref{eq:app_chi2_expanded}. Rearranging:
\begin{equation}
    \chi^2(P \| Q) = \exp\!\left(D_2(P \| Q)\right) - 1.
    \label{eq:app_chi2_from_d2}
\end{equation}

\paragraph{Property 2: Monotonicity in $\alpha$.}
Previous studies \citep{van_Erven_2014} have demonstrated that $D_\alpha(P\|Q)$ is non-decreasing in $\alpha$. The KL divergence is recovered in the limit $\alpha \to 1$:
\begin{equation}
    D_{\mathrm{KL}}(P \| Q) = \lim_{\alpha \to 1}\, D_\alpha(P \| Q).
\end{equation}
Since $\alpha = 2 > 1$, monotonicity immediately gives:
\begin{equation}
    D_2(P \| Q) \;\geq\; D_{\mathrm{KL}}(P \| Q).
    \label{eq:app_d2_geq_kl}
\end{equation}

Combining Eq.~\eqref{eq:app_chi2_from_d2} and Eq.~\eqref{eq:app_d2_geq_kl}, and using the monotonicity of the exponential:
\begin{equation}
    \mathrm{Var}_Q[\rho]
    = \chi^2(P \| Q)
    = \exp(D_2) - 1
    \;\geq\; \exp(D_{\mathrm{KL}}) - 1.
    \label{eq:app_var_lower_bound}
\end{equation}

\subsection{Effective Sample Size Decay}
\label{app:ess_derivation}

The effective sample size (ESS)~\citep{kong1992note} quantifies how many of $n$ importance-weighted samples are effectively contributing to an estimate. Formally:
\begin{equation}
    \mathrm{ESS} = \frac{n}{1 + \mathrm{Var}_{Q}[\rho]}.
    \label{eq:app_ess_def}
\end{equation}
When $\pi_\theta = \pi_\mu$ (i.e., $\rho \equiv 1$, $\mathrm{Var}[\rho] = 0$), all samples are equally useful and $\mathrm{ESS} = n$. As $\mathrm{Var}[\rho]$ grows, more samples are effectively wasted and the ESS shrinks.

Substituting the lower bound from Eq.~\eqref{eq:app_var_lower_bound}:
\begin{equation}
    \mathrm{ESS}
    = \frac{n}{1 + \mathrm{Var}[\rho]}
    \leq \frac{n}{1 + \exp(D_{\mathrm{KL}}) - 1}
    = \frac{n}{\exp(D_{\mathrm{KL}})}.
    \label{eq:app_ess_step1}
\end{equation}
This shows that per-sample effective contribution decays exponentially with the KL divergence between the two policies.

\section{Environment and Model Details}
\label{app:env}

Table~\ref{tab:env_details} summarizes the eight evaluation environments. We use three model scales across experiments, chosen to match the computational cost of each task.

\begin{table}[ht]
\scriptsize
\caption{Environment and model configurations. ``Max Actions'' denotes the maximum number of assistant turns per episode. ``Seq Len'' is the maximum sequence length (prompt + all turns). ``Config'' refers to the hyperparameter configuration (\cref{app:config_a}).}
\label{tab:env_details}
\begin{tabular}{llcccccl}
    \toprule
    \textbf{Environment} & \textbf{Type} & \textbf{Model} & \textbf{GPUs} & \textbf{Seq Len} & \textbf{Max Actions} & \textbf{Config} & \textbf{Reward} \\
    \midrule
    NQ Search & LLM & Qwen2.5-7B & 8 & 12800 & 5 & Default & Binary (EM) \\
    AIME & LLM & Qwen2.5-7B & 8 & 4096 & 3 & A & Binary \\
    Sokoban Simple & LLM & Qwen2.5-0.5B & 2 & 2048 & 10 & A & $[-1, +3]$ \\
    Sokoban Hard & LLM & Qwen2.5-0.5B & 2 & 2048 & 10 & A & $[-1, +3]$ \\
    FrozenLake (LLM) & LLM & Qwen2.5-0.5B & 2 & 2048 & 10 & A & Binary \\
    CliffWalking & LLM & Qwen2.5-0.5B & 2 & 2048 & 200 & A & $[-\infty, 0]$ \\
    GSM8K & LLM & Qwen2.5-0.5B & 2 & 4096 & 3 & A & Binary \\
    \midrule
    FrozenLake (VLM) & VLM & Qwen2.5-VL-3B & 4 & 4096 & 10 & Default & Binary \\
    GeoQA & VLM & Qwen2.5-VL-3B & 4 & 4096 & 3 & Default & Binary \\
    \bottomrule
\end{tabular}
\end{table}

\paragraph{NQ Search.}
An agentic retrieval-augmented QA task on Natural Questions~\citep{jin2025searchr1}. The agent alternates between reasoning (\texttt{<think>}), issuing search queries (\texttt{<search>}), and producing a final answer (\texttt{<answer>}). Retrieved documents are provided via a FAISS index with E5 embeddings ($k{=}3$ passages per query). Reward is binary exact-match (EM), with case/punctuation/article normalization. The long sequence length (12800 tokens) accommodates multi-turn search context.

\paragraph{AIME.}
A math competition task based on the American Invitational Mathematics Examination. Each problem requires an integer answer in the range 000--999. The agent has up to 3 attempts per problem; if the first answer is incorrect, it receives feedback and may try again. Reward is binary (correct/incorrect). We use Qwen2.5-7B-Instruct with Config~A ($\tau{=}1000$) on 8 GPUs.

\paragraph{Sokoban.}
A box-pushing puzzle on a grid where the agent must push all boxes onto target positions~\citep{wang2025roll}. Actions are irreversible---a misplaced box cannot be pulled back. We evaluate two difficulty levels: \emph{Simple} ($6{\times}6$ grid, 1~box) and \emph{Hard} (larger layouts, more boxes). Score ranges from $-1$ (all boxes misplaced) to $+3$ (all on target), making it a dense reward signal with a meaningful performance gradient.

\paragraph{FrozenLake (LLM).}
A text-only navigation task on a $4{\times}4$ grid with slippery ice. The agent receives a text description of the grid state and must navigate to the goal while avoiding holes. Slippery dynamics mean the agent moves in the intended direction with probability $1/3$ and slides sideways otherwise, making the task stochastic. Used for both main evaluation and ablation studies.

\paragraph{CliffWalking.}
A $4{\times}12$ grid navigation task where the agent must reach the goal while avoiding a cliff along the bottom row. Falling off the cliff incurs a large negative reward ($-100$) and resets the agent. The optimal policy achieves score~0. This serves as a control: the task is simple enough that on-policy training solves it quickly, and replay is not expected to help.

\paragraph{GSM8K.}
Grade-school math word problems~\citep{cobbe2021gsm8k} requiring multi-step arithmetic reasoning. The base Qwen2.5-0.5B model already achieves ${>}$93\% accuracy on this task, so it serves as a near-saturated control where replay cannot provide meaningful improvements.

\paragraph{FrozenLake (VLM).}
The visual counterpart of the text FrozenLake: the $4{\times}4$ grid is rendered as an RGB image, and the VLM agent must interpret the visual state to navigate. This tests whether \method{} generalizes to multimodal settings where observations are images rather than text.

\paragraph{GeoQA.}
A geometry question-answering task where the VLM interprets geometric diagrams and solves problems requiring spatial reasoning. Binary reward (correct/incorrect).

\section{Engineering Insights: Hyperparameter Configuration for Off-Policy LLM RL}
\label{app:config_a}

During development, we discovered that standard on-policy hyperparameters are unsuitable for off-policy training with a replay buffer. We summarize the key findings below, as they may benefit practitioners adopting experience replay for LLM RL.

\paragraph{Advantage Clipping is Critical.}
The default advantage clipping threshold in many LLM RL implementations is large ($\epsilon_{\text{clip}} = 20$), effectively disabling clipping. For on-policy training this is harmless, but with replay data, stale trajectories can produce extreme advantage values that destabilize training. We found that reducing the clip to $\epsilon_{\text{clip}} = 0.2$---the same value used for the PPO ratio clip---was essential for stable off-policy learning. With $\epsilon_{\text{clip}} = 20$, replay-augmented training frequently diverged; with $\epsilon_{\text{clip}} = 0.2$, it was consistently stable across all environments.

\paragraph{KL Regularization Should Be Disabled.}
On-policy methods commonly use KL divergence penalties ($\beta_{\text{KL}} = 0.05$--$0.1$) to prevent the policy from deviating too far from the reference model. However, with off-policy replay, the KL penalty interacts poorly with stale data: old trajectories already have large KL divergence from the reference, and penalizing this further suppresses useful gradient signal from replay. We found that setting $\beta_{\text{KL}} = 0$ (and disabling the adaptive KL controller) improved performance across all environments when using a replay buffer.

\paragraph{Entropy Bonus is Unnecessary.}
Similarly, the entropy loss coefficient ($\lambda_{\text{ent}} = 0.01$) commonly used for exploration in on-policy training was counterproductive with replay. The replay buffer itself provides exploration diversity; adding an entropy bonus on top led to overly stochastic policies.

Table~\ref{tab:config_comparison} summarizes these differences. We refer to the off-policy-optimized configuration as ``Config~A'' throughout our experiments.

\begin{table}[ht]
\centering
\caption{Hyperparameter comparison between standard on-policy settings and our off-policy-optimized Config~A.}
\label{tab:config_comparison}
\small
\begin{tabular}{lcc}
    \toprule
    \textbf{Parameter} & \textbf{On-Policy (Default)} & \textbf{Config~A (Off-Policy)} \\
    \midrule
    Advantage clip $\epsilon_{\text{clip}}$ & 20 & \textbf{0.2} \\
    KL loss enabled & \checkmark & $\times$ \\
    KL coefficient $\beta_{\text{KL}}$ & 0.05--0.1 & \textbf{0.0} \\
    Init KL coefficient & 0.1 & \textbf{0.0} \\
    Entropy coefficient $\lambda_{\text{ent}}$ & 0.01 & \textbf{0.0} \\
    PPO epochs & 1 & 1 \\
    Max gradient norm & 1.0 & 1.0 \\
    Advantage estimator & REINFORCE++ & REINFORCE++ \\
    Whiten advantages & \checkmark & \checkmark \\
    \bottomrule
\end{tabular}
\end{table}

\noindent\textbf{Note}: Config~A was developed on FrozenLake (LLM, 0.5B) and subsequently validated on Sokoban, CliffWalking, and GSM8K without modification---all 0.5B experiments in this paper use Config~A for both baseline and replay runs. The NQ Search (7B) and VLM (3B) experiments use the default on-policy settings, as larger models appear more robust to these hyperparameter choices.

\section{Detailed Experimental Configurations}
\label{app:configs}

Table~\ref{tab:training_config} lists the full training configuration for each experiment family. All experiments use the ROLL framework~\citep{wang2025roll} with REINFORCE++ as the policy gradient algorithm, DeepSpeed for training, and vLLM for inference.

\begin{table}[ht]
\centering
\caption{Training configurations across experiment families. ``GA'' = gradient accumulation steps, ``BS/dev'' = per-device batch size. All experiments use learning rate $10^{-6}$.}
\label{tab:training_config}
\small
\begin{tabular}{lccccccc}
    \toprule
    \textbf{Experiment} & \textbf{Model} & \textbf{GPUs} & \textbf{BS/dev} & \textbf{GA} & \textbf{DeepSpeed} & \textbf{Max Steps} & \textbf{Config} \\
    \midrule
    NQ Search & 7B & 8 & 1 & 16 & ZeRO-3 & 300 & Default \\
    AIME & 7B & 8 & 2 & 32 & ZeRO-2 & 400 & A \\
    Sokoban & 0.5B & 2 & 4 & 32 & ZeRO-2 & 400 & A \\
    FrozenLake (LLM) & 0.5B & 2 & 4 & 32 & ZeRO-2 & 400 & A \\
    CliffWalking & 0.5B & 2 & 4 & 32 & ZeRO-2 & 400 & A \\
    GSM8K & 0.5B & 2 & 4 & 32 & ZeRO-2 & 400 & A \\
    \midrule
    FrozenLake (VLM) & VL-3B & 4 & 1 & 32 & ZeRO-2 & 100 & Default \\
    GeoQA & VL-3B & 4 & 1 & 32 & ZeRO-2 & 100 & Default \\
    \bottomrule
\end{tabular}
\end{table}

Table~\ref{tab:replay_config} details the replay buffer configuration. All replay-enabled experiments use a trajectory-level buffer with FIFO eviction and the same core priority parameters.

\begin{table}[ht]
\centering
\caption{Replay buffer configurations. All experiments use trajectory-level sampling, FIFO eviction, and priority exponent $\alpha{=}0.6$.}
\label{tab:replay_config}
\small
\begin{tabular}{lccccc}
    \toprule
    \textbf{Method} & \textbf{Capacity} & \textbf{$K$} & \textbf{Priority} & \textbf{Age Decay $\tau$} & \textbf{IS ($\beta$)} \\
    \midrule
    On-Policy & --- & --- & --- & --- & --- \\
    Standard PER & 50K & 2 & $|r| + \epsilon$ & $\infty$ (disabled) & \checkmark~(0.4) \\
    \method{} (default) & 50K & 2 & $|r| + \epsilon$ & 500 & $\times$ \\
    \method{} ($\tau{=}1000$) & 50K & 2 & $|r| + \epsilon$ & 1000 & $\times$ \\
    \method{} ($\tau{=}1500$) & 50K & 2 & $|r| + \epsilon$ & 1500 & $\times$ \\
    \method{} ($\tau{=}500$, IS) & 50K & 2 & $|r| + \epsilon$ & 500 & \checkmark~(0.4) \\
    \bottomrule
\end{tabular}
\end{table}

\paragraph{Common Settings.}
All experiments share: rollout batch size 128, validation batch size 128, seed 42, cosine learning rate schedule with 10-step warmup, BF16 mixed precision, and FlashAttention-2. Inference uses vLLM with $\text{top-}p{=}0.99$, $\text{top-}k{=}100$, and temperature~0.99.

\paragraph{GPU Layout.}
For 2-GPU experiments (0.5B models): GPU~0 runs training (DeepSpeed) and reference model inference; GPU~1 runs vLLM inference. For 4-GPU experiments (VLM 3B): GPUs~0--1 run training and reference; GPUs~2--3 run vLLM. For 8-GPU experiments (NQ Search 7B): GPUs~0--3 run training (ZeRO-3); GPUs~4--7 run vLLM.

\end{document}